\documentclass{article}
\usepackage{ijcai22}
\usepackage{colortbl}
\usepackage{times}
\usepackage{soul}
\usepackage[hidelinks]{hyperref}
\usepackage[utf8]{inputenc}
\usepackage[small]{caption}
\usepackage{amsmath}
\DeclareMathOperator*{\argmax}{argmax} 
\DeclareMathOperator*{\argmin}{argmin}
\usepackage{amsthm}

\usepackage{amssymb}

\usepackage[noend]{algpseudocode}
\usepackage{algorithm}
\usepackage{algorithmicx}
\usepackage{bm}
\usepackage{amsmath,amsfonts}
\usepackage{float}
\usepackage{subfigure} 
\usepackage{booktabs}
\usepackage{caption}
\usepackage{url}
\usepackage{footnote}
\usepackage{multirow}
\usepackage{graphicx}
\usepackage{textcomp}
\usepackage{xcolor}
\usepackage{color}
\usepackage{enumitem}
\def\BibTeX{{\rm B\kern-.05em{\sc i\kern-.025em b}\kern-.08em
    T\kern-.1667em\lower.7ex\hbox{E}\kern-.125emX}}

\usepackage{array}
\newcolumntype{?}{!{\vrule width 1.5pt}}

\newtheorem{theorem}{Theorem}
\newtheorem{definition}{Definition}
\newtheorem{corollary}{Corollary}
\newtheorem{proposition}{Proposition}

\title{OptIForest: Optimal Isolation Forest for Anomaly Detection}

\author{
Haolong Xiang$^1$
\and
Xuyun Zhang$^{1*}$\and
Hongsheng Hu$^{2}$\and
Lianyong Qi$^3$\and
Wanchun Dou$^4$\and
Mark Dras$^1$\and
Amin Beheshti$^1$\And
Xiaolong Xu$^{5*}$
\affiliations
$^1$Macquarie University\\
$^2$CSIRO’s Data61\\
$^3$Qufu Normal University\\
$^4$Nanjing University\\
$^5$Nanjing University of Information Science and Technology 
\emails
haolong.xiang@hdr.mq.edu.au,
\{xuyun.zhang, mark.dras, amin.beheshti\}@mq.edu.au,
Hongsheng.Hu@data61.csiro.au,
lianyongqi@gmail.com,
douwc@nju.edu.cn,
xlxu@nuist.edu.cn
}

\begin{document}

\maketitle

\begin{abstract}
    Anomaly detection plays an increasingly important role in various fields for critical tasks such as intrusion detection in cybersecurity, financial risk detection, and human health monitoring. A variety of anomaly detection methods have been proposed, and a category based on the isolation forest mechanism stands out due to its simplicity, effectiveness, and efficiency, e.g., iForest is often employed as a state-of-the-art detector for real deployment. While the majority of isolation forests use the binary structure, a framework LSHiForest has demonstrated that the multi-fork isolation tree structure can lead to better detection performance. However, there is no theoretical work answering the fundamentally and practically important question on the optimal tree structure for an isolation forest with respect to the branching factor. In this paper, we establish a theory on isolation efficiency to answer the question and determine the optimal branching factor for an isolation tree. Based on the theoretical underpinning, we design a practical optimal isolation forest OptIForest incorporating clustering based learning to hash which enables more information to be learned from data for better isolation quality. The rationale of our approach relies on a better bias-variance trade-off achieved by bias reduction in OptIForest. Extensive experiments on a series of benchmarking datasets for comparative and ablation studies demonstrate that our approach can efficiently and robustly achieve better detection performance in general than the state-of-the-arts including the deep learning based methods.
    
\end{abstract}

\section{Introduction}

Detection of anomalies (also known as outliers) is an important machine learning task to capture abnormal patterns or sparse observations in data that collide with the majority showing the expected behaviours ~\cite{pang2021deep}, and has been deployed in a broad range of fields for critical applications such as intrusion detection in cybersecurity, financial risk detection, and human or device health monitoring \cite{ahmed2016survey,chen2022antibenford,fernando2021deep}. While anomalies often take a very small portion of the data or appear infrequently, failing to catch them in a timely manner for further actions can result in severe consequences such as cascading failures in manufacture and deaths in healthcare. Therefore, a variety of unsupervised anomaly detection methods \cite{ruff2021unifying}, from shallow to deep, have been proposed for different types of anomalies and data types. Note that since collecting labels for data is usually difficult and expensive, especially for anomalies, we herein focus on unsupervised anomaly detection which is more commonly-used in practice. Recently, several deep learning based detection methods, e.g., RDP \cite{wang2021unsupervised} and REPEN \cite{pang2018learning}, have shown the advantages of having feature representation learning for anomaly detection ~\cite{pang2021deep}. But this does not mean that the traditional (shallow) detection methods will be obsolete because deep neural networks have their own intrinsic limitations such as high computational cost, poor explainability, and difficulty in hyperparameter tuning \cite{li2022ecod}, or a shallow model can have comparable performance but cost much less, e.g., a very recent work ECOD \cite{li2022ecod} just use the statistic analysis of the tail event of a distribution to achieve the good performance over various benchmark datasets. Instead, traditional models can be preferred in specific scenarios, e.g., edge computing where computational resources are limited and medical research where high explainability is in demand. Moreover, these methods can work together with the feature representations learned from deep learning for better performance, e.g., a recent work \cite{xu2022deep} has shown such an example that iForest is successfully used together with deep learning.

Benefiting from ensemble learning~\cite{aggarwal2015theoretical}, a category of detection methods based on the isolation forest mechanism \cite{hariri2019extended} stands out of the shallow models due to its excellent simplicity, effectiveness, and efficiency, therefore being really promising for big data anomaly detection. The basic idea is to randomly and recursively partition relatively small samples drawn from a dataset until all data instances are isolated, which produces a forest of isolation trees. Anomaly scores can be derived from an isolation forest based on an observation that anomalies often have shorter path lengths due to the ease of isolating them from others. As the first instance, iForest \cite{liu2008isolation} has been widely recognised in academia and deployed in real applications, e.g., it has been included in \textit{scikit-learn}, a commonly-used machine learning library in Python. While most of isolation forests following iForest use the binary tree structure for data isolation, a framework LSHiForest \cite{zhang2017lshiforest} producing multi-fork isolation trees with the use of similarity hash functions has demonstrated better detection performance. An interesting question arises naturally: What is the optimal branching factor for an isolation tree? However, there is little theoretical work answering this fundamentally important question, and practically this wide gap can hamper the further development of isolation forest based anomaly detection methods.

In this paper, we investigate this interesting problem and establish a theory on the structure optimality of an isolation tree with respect to the branching factor by introducing the notion of isolation efficiency. A constrained optimisation problem is formulated to solve the optimality problem, and an interesting finding is that the optimal branching factor is $e$ (Euler's number), rather than $2$ which is commonly-used in existing methods. Based on the theoretical foundation, a practical optimal isolation forest named OptIForest is proposed for efficient and robust anomaly detection. Specifically, we adapt clustering based learning to hash in OptIForest to let the isolation process optimised with more information learned from data. With two key observations on ensemble learning and anomaly distribution in an isolation tree, we design controllable initialisation for agglomerative hierarchical clustering, enabling OptIForest to achieve a better bias-variance trade-off via bias reduction from learning and improve the computation efficiency. Extensive experiments are performed on a suite of benchmarking datasets in our ablation and comparative studies. The results validate our proposed theory on optimal isolation forest, and also show that with respect to detection performance, OptIForest can generally outperform the state-of-the-arts including the deep anomaly detection methods while maintaining a high computation efficiency. The source code is available at \url{https://github.com/xiagll/OptIForest}.

The main contributions of our work are threefold, summarised as follows: (1) We are the first to formally investigate the optimality problem of isolation tree structure with respect to the branching factor and establish a theory on the optimal isolation forest, offering a theoretical understanding for the effectiveness of the isolation forest mechanism. (2) We innovatively propose a practical optimal isolation forest OptIForest that can enhance both detection performance and computational efficiency by designing a tailored clustering based learning to hash for a good bias-variance trade-off. (3) Results from extensive experiments support our theory well and validate the effectiveness and efficiency of our approach, as well as the advantages over the state-of-the-arts.

\section{Related Work}
Various methods for unsupervised anomaly detection have emerged, including distance-based, density-based, statistical, ensemble-based, and deep learning-based methods~\cite{chandola2009anomaly,pang2021deep}. This section delves into the crucial methods related to our study.

\noindent \textbf{Deep Learning-based Methods.}~~Recently, deep neural networks are widely explored for anomaly detection, particularly on complex data types~\cite{zong2018deep,zavrtanik2021reconstruction}. Techniques like generative adversarial networks (GANs)~\cite{liu2019generative}, AutoEncoders~\cite{chen2017outlier}, and reinforcement learning have been leveraged to enhance the detection performance~\cite{pang2021toward}. For example,  a random distance-based anomaly detection method called REPEN is proposed in \cite{pang2018learning}, where learning low-dimensional representation in random subsample is optimised. While these deep methods can have high accuracy with learning feature representation \cite{zha2020meta,zhang2021elite}, they often suffer from issues like expensive computations, complex hyperparameter tuning, etc.

\noindent \textbf{Ensemble-based Methods.}~~To achieve robust detection, classical anomaly detection methods are integrated with ensemble learning, e.g., ensemble LOF~\cite{zimek2013subsampling}, isolation using Nearest Neighbour Ensemble (iNNE)~\cite{bandaragoda2014efficient}, and average $k$-NN distance ensemble~\cite{aggarwal2015theoretical}. These methods suffer from heavy computational cost when handling big data. A simpler but more effective method is iForest \cite{liu2012isolation} which leverages the principle that anomalies are more likely to be isolated from others. This pioneering work shows the strong ability of the isolation forest mechanism and has been widely adopted in real applications. A sequence of work on isolation forests have been proposed to mitigate the shortcomings in iForest, e.g., SCiForest \cite{liu2010detecting} addresses the failure of detecting axis-parallel anomalies and local anomalies by using a optimsation strategy. Another interesting work is LSHiForest which produces multi-fork isolation trees with the use of LSH (locality-sensitive hashing) functions that can hash similar data into the same hash value. The work formally understands the distance metric underlying iForest and enables the isolation forest mechanism widely applicable to any data types where an LSH family can be defined. But why a multi-fork isolation forest performs better has not been tackled in LSHiForest. Deep isolation forest \cite{xu2022deep} has been recently proposed to incorporate isolation forest with random feature representation, extending iForest to complex data types. However, these works still fail to learn information properly from data to reduce the bias of a base detector for a better bias-variance trade-off.
 


\section{Preliminaries and Problem Statement}

\noindent \textbf{LSHiForest Framework \cite{zhang2017lshiforest}.} Our approach is built on top of the LSHiForest Framework which is an effective and generic anomaly detection framework with the forest isolation mechanism. While the algorithmic procedural and the way of deriving anomaly scores in LSHiForest are quite similar to that in iForest, a key difference is that LSHiForest makes use of a hash function to determine the branching when isolating data in the recursive tree construction process, i.e., data instances with the same hash value go into the same branch. As a result, the isolation trees can be multi-fork, which is significantly different from the binary case in iForest. Since the hash functions are drawn from an LSH (Locality-Sensitive Hashing) family \cite{bawa2005lsh}, the data instances falling into the same branch are similar to each other with provably high probability and the isolation process is more natural than binary splitting. Moreover, the LSHiForest framework has high applicability to work with any distance metric with an LSH family. In fact, iForest and its variant have been proved to be two specific instances of the framework with less commonly-used distance metrics. Besides, other instances of the framework with specific distance metrics like Angular distance, Manhattan ($\ell_1$) distance, and Euclidean ($\ell_2$) distance have been implemented as well, and the instance with Euclidean distance (denoted as L2SH) shows a very efficient, robust, and accurate detection performance. Thus, we are motivated to use LSHiForest and its L2SH instance as the foundation of our approach given their prominent features and excellent performance.

\noindent \textbf{Problem Statement.}~~Although multi-fork isolation trees can be elegantly constructed in LSHiForest and better detection performance can be achieved, the theoretical understanding of this phenomenon is still missing. Accordingly, an interesting question arises naturally: What is the optimal branching factor for an isolation tree? Since the branching factor influences the tree structure which can be regarded as a parameter of the detection model, answering this question is fundamentally important to understand how the branching factor affects the performance of isolation forest based anomaly detection methods, but intrinsically a challenging task. Fig.~\ref{fig_isolation_example} shows an example of isolating nine data instances with three different tree structures, including $9$-fork tree, binary tree, and ternary tree. It is hard to intuitively tell which structure is the best for isolation, so how to address this problem is non-trivial.


\begin{figure}
  \centering
  \includegraphics[width=0.4\textwidth]{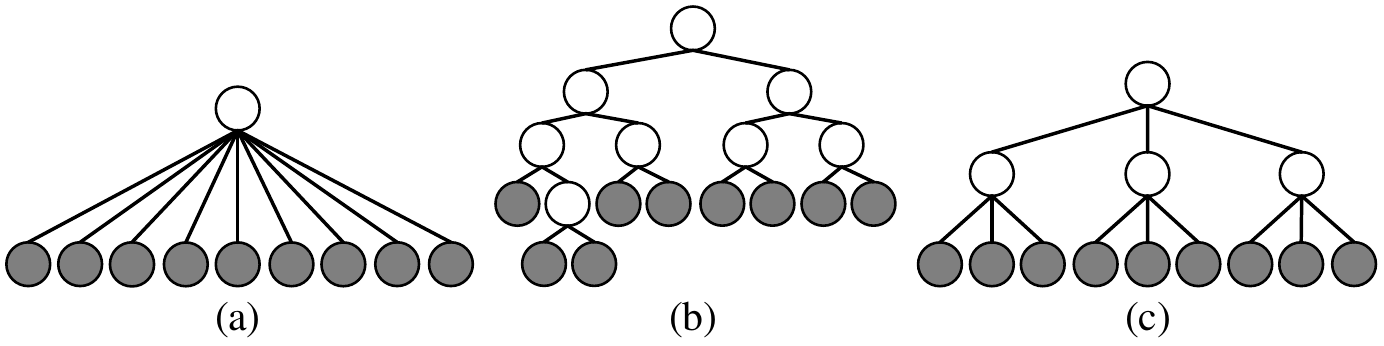}
  \captionsetup{font={small}}
  \caption{Isolating $9$ data instances with different tree structures.}
  \label{fig_isolation_example}
\end{figure}

Besides the tree structure, another challenge imposed on existing isolation forest based anomaly detection methods is how much information they should learn from the data to facilitate data partition at each internal node. For instance, iForest learns the minimum and maximum of the selected feature to determine the random splitting value, and its variant SCiForest\cite{liu2010detecting} learns more information to determine the splitting hyperplane and gains better performance. But it is worth noting that more learning does not means better detection performance while the bias of a single base detector can be reduced, according to the bias-variance trade-off theory in ensemble learning \cite{aggarwal2015theoretical}. For example, the recent work \cite{xu2022deep} actually reports that full learning for the neural networks has poorer detection performance than the random weight initiation. Because LSH hash functions used in LSHiForest are data-independent, no learning is incorporated in LSHiForest for data isolation. Thus, it is conjectured that the detection performance could be further improved if an isolation tree is constructed with the learning to hash technique \cite{wang2017survey} which produces hash values based on learning information from the data. But learning to hash usually incurs higher computational cost than LSH. Therefore, non-trivial effort is required for the design of an anomaly detector based on isolation forest with learning to hash, to target both high computational efficiency and a good bias-variance trade-off.

\section{Methodology}

\subsection{Optimal Isolation Forest}
\label{optimaltree}
To answer the research question stated above about the optimal branching factor of an isolation tree, this section formulates the problem to derive the solution and discusses its practical implementation. Let $T$ represents an isolation tree, with the branching factor $v$, the tree depth $d$, and the number of leaf nodes $\psi$. For the purpose of theoretical analysis, we can reasonably allow $v$, $d$, and $\psi$ to take real values rather than just integers. Actually, $v$ and $d$ can be regarded as average branching factor of the internal nodes and the average height of the leaf nodes, respectively. To study how the branching factor influences the isolation performance of an isolation tree, we assume that $T$ is a perfect tree where all internal nodes have $v$ children and all leaves have the same depth $d$. Note that the isolation tree structure may be totally different in practice, but the assumption is the average case which is reasonable in the ensemble learning setting. Given $T$ is perfect, we can have the following relationship:
\begin{equation}
    \psi=v^d.
\end{equation}

\begin{definition}[Isolation Capacity]
\label{definition-icapacity}
The maximum number of data instances an isolation tree can isolate is defined as the isolation capacity of the tree, denoted as $\psi$, which is also the number of leaf nodes. 
\end{definition}

Given an isolation capacity, it can be interestingly observed that the width (controlled by $v$) and the depth (controlled by $d$) of the tree compete with each other, i.e., if the branching factor is smaller, the tree has to be deeper, or vice versa. So, we can have the following definition to capture the overall effect of branching factor and tree depth.

\begin{definition}[Isolation Area]
\label{definition-icost}
The isolation area of an isolation tree, denoted as $\phi$, is defined as the product of the branching factor $v$ which controls the width of a tree and the depth $d$, i.e.,
\begin{equation}
    \phi=v{\cdot}d.
\end{equation}
\end{definition}

Note that isolation trees of the same isolation area can achieve different isolation capacities. For example, given an isolation area $\phi=6$, a perfect binary tree (i.e., $v=2$ and $d=3$) has the isolation capacity $\psi=2^3=8$, but a perfect ternary tree (i.e., $v=3$ and $d=2$) has the isolation capacity $\psi=3^2=9$. The latter case seems more efficient, and we can further define the concept of isolation efficiency as follows.

\begin{definition}[Isolation Efficiency]
The isolation efficiency of an isolation tree, denoted as $\eta$, is defined as the quotient of the isolation capacity divided by the isolation area, i.e., 
\begin{equation}
    \eta=\frac{\psi}{\phi}.
\end{equation}
\end{definition}

Isolation efficiency fundamentally affects the anomaly detection performance of isolation forests because it is associated with the hardness of distinguishing data instances isolated by a tree. Specifically, higher isolation efficiency leads to stronger distinguishability. Continuing the example above, the isolation efficiency of the binary tree is $\eta = 8/6 \simeq 1.33$ and that of the ternary tree is $\eta = 9/6=1.5$. Fig.~\ref{fig_isolation_example}(c) shows that the ternary tree can isolate all the data, while Fig.~\ref{fig_isolation_example}(b) shows that a binary tree with only 3 layers fails to do so.


As the branching factor indeed affects the detection performance in terms of the analysis presented above, it is of both theoretical and practical importance to study the problem of the optimal branching factor. With the concepts introduced above, we can formulate the problem of identifying the optimal branching factor $v^*$ as a constrained optimisation problem as follows.

\begin{equation}
\begin{aligned}
v^*= & ~~ \argmax_{v} ~~ \eta(v, d)=\frac{v^d}{vd},\\
 & \textrm{s.t.} ~~ vd=\Phi,\\
\end{aligned}
\end{equation}

{\noindent}where $\eta(v, d)$ is the function of isolation efficiency with respect to branching factor $v$ and tree depth $d$, and $\Phi$ is a constant number representing the fixed isolation area. By solving the optimisation problem, we can have the following theorem.

\begin{theorem}
\label{theorem} 
An isolation tree $T$ has the highest isolation efficiency when its branching factor $v=e$, where $e$ is the Euler's number with numerical values around $2.718$.
\end{theorem}

The theorem can be proved with solving the optimisation problem in (4). Please see Appendix~\ref{proof1} for mode details, where how the isolation efficiency changes with respect to branching factor with a fixed isolation area is also illustrated. This interesting result reveals that the commonly-used binary tree is not the optimal for the forest isolation mechanism, and explains the better detection accuracy and robustness gained by the multi-fork LSHiForest instances like L2SH.

\begin{definition}[Optimal Isolation Tree]
\label{definition-OptITree}
An isolation tree with branching factor $v=e$ is defined as an optimal isolation tree.
\end{definition}
\begin{definition}[Optimal Isolation Forest]
\label{definition-OptIForest}
A forest consisting of a set of optimal isolation trees is defined as an optimal isolation forest.
\end{definition}

Although the optimal isolation tree is theoretically promising, unfortunately there is no real $e$-fork branching can be constructed directly in reality. As a practically viable way, we can build an isolation tree with the average branching factor equal to $e$. To achieve this, we need to generate the branching factors during the tree construction. Let a random variable $V$ denote the branching factor for an isolation tree $T$, with the sample space $\{v~|~v\in\mathbb{Z}~ \& ~v~{\geq}~2\}$. Let $\mathcal{D}$ be a distribution with the probability of taking value $v$ being: $\textrm{Pr}(V=v)=p_v$. We can generate the branching factors from the distribution $\mathcal{D}$ if it satisfies the following condition:
\begin{equation}
\label{expectation:condition}
    \mathbb{E}(V)=\sum_{v=2}^{+\infty}v{\cdot}p_v=e.
\end{equation}

As the branching factor $2$ is the only one less than $e$, there should be a lower bound for $p_2$ and upper bounds for $p_v$, $v>2$, to satisfy the condition mentioned above.

\begin{theorem}
\label{theorem:uppperbound}
To satisfy the condition in Eq.(\ref{expectation:condition}), the probability of having a branching factor $V{\geq} v$, $v>2$, should have the following upper bound:
\begin{equation}
    \textrm{Pr}(V{\geq}v) ~{\leq}~ \frac{(e-2)}{(v-2)}.
\end{equation}
\end{theorem}

The theorem can be proved by letting $p_i=0$ for $2<i<v$ and leveraging the inequality $\sum_{i=v}^{+\infty}ip_i ~{\geq}~ v\sum_{i=v}^{+\infty}p_i$. Please see Appendix~\ref{proof2} for more details. The upper bound result can give us an intuition on how the probability decreases when the branching factor increases, e.g., the probability is $e-2 \simeq 0.718$ for $V{\geq}3$, $\frac{(e-2)}{2} \simeq 0.359$ for $V{\geq}4$, and $\frac{(e-2)}{3} \simeq 0.239$ for $V{\geq}5$.

\begin{corollary}
To satisfy the condition in Eq.(\ref{expectation:condition}), the probability of sampling the branching factor $V=2$ should follow $p_2 \geq 3-e \simeq 0.282$.
\end{corollary}

This corollary can be simply derived from Theorem \ref{theorem:uppperbound} when $V{\geq}3$, i.e., $p_2 = 1-\textrm{Pr}(V{\geq}3) ~{\leq}~ 1-(e-2)=3-e$.

There are many possible concrete distributions for $\mathcal{D}$ to be specified, either finite or infinite. For example, by just using binary and ternary branches, we can have a simple finite distribution as follows: $p_2=3-e$, $p_3=e-2$, and $p_i=0$, $i~{\geq}~4$. For the infinite case, the probability is often a function of the branching factor. The work in \cite{russell1991estimating} derives a distribution $p_v=\frac{(v-1)}{(v!)}$, $v~{\geq}~2$, from a stochastic technique. Besides, we can have another infinite distribution $p_v=\frac{(e-1)^2}{2e-1}e^{2-v}$, $v~{\geq}~2$, to satisfy the condition in Eq.(\ref{expectation:condition}), as proved in Appendix~\ref{proof3}. Understanding the distributions of branching factors and the probability bounds can facilitate the decision-makings when one tries to design a more practical optimal isolation forest, as shown in the following section.

\subsection{OptIForest: Practical Detector Design with Clustering based Learning to Hash}
In this section, we investigate how to implement a practical optimal isolation tree. Since the data-independent LSH hash functions cannot exactly produce a specified number of hash values, we have to leverage learning to hash to achieve this for a specified branching factor. Moreover, learning to hash can capture more information from the data to benefit anomaly detection potentially. There are many types of learning to hash \cite{wang2017survey}. Given the higher accuracy and less quantification loss, the non-parametric hash function $h(\cdot)$ based on nearest vector assignment is adopted herein, i.e.,
\begin{equation}
    h(\boldsymbol{x})=\argmin_{k{\in}\{1, \cdots, v\}} ||\boldsymbol{x}-\boldsymbol{c}_k||,
\end{equation}
{\noindent}where $\{\boldsymbol{c}_1, \cdots, \boldsymbol{c}_v\}$ is a set of centres of data partitions. Clustering algorithms are employed to compute the centres. 

As our goal is to arrange the clusters into a natural hierarchy to form an isolation tree, agglomerative hierarchical clustering is adopted in our approach. The clustering method treats each data instance as a cluster initially and merges similar clusters sequentially until a single cluster is left, forming a hierarchical tree in a bottom-up fashion. While good clustering quality with less data distortion can be obtained, the high computational complexities are often regarded as a downside. With a cubic time complexity, clustering on a small dataset of a fixed size (e.g., 256 in iForest, and 1024 in LSHiForest) still seems too time-consuming. Straightforward adoption of the clustering technique is neither effective nor efficient. 

Fortunately, we find that the clustering technique can be optimised based on two key observations in an isolation forest and design a more efficient clustering method. One observation is that anomalies are often located at upper levels of an isolation tree while normal data at lower levels. The quality of clustering at upper levels is more sensitive to anomaly detection than that at lower levels. The other observation is that no learning or full learning in existing methods can lead to poorer detection performance. Therefore, we can let the upper levels of an isolation tree learn more from data to have better clustering quality, while letting the lower levels learn less, aiming to improve detection performance and save computational cost simultaneously. 

The basic idea of our approach is to use an isolation tree efficiently produced in LSHiForest to initialise the clusters for agglomerative hierarchical clustering so that the clustering process begins with bigger initial clusters rather than the ones with a single data instance. This can maintain the high clustering quality for upper levels of the hierarchy and save much computational cost because the majority of computation occurs at lower levels. The specific steps for constructing an optimal tree $T_{Opt}$, are outlined in Algorithm~\ref{alg1}, and details are discussed subsequently.
\begin{algorithm}
	\caption{Constructing an Optimal Isolation Tree}
	\label{alg1}
	\begin{algorithmic}[1]
		\renewcommand{\algorithmicrequire}{\textbf{Input:}}
		\renewcommand{\algorithmicensure}{\textbf{Output:}}
		\Require A dataset (sample) $D$ of size ${\psi}$, and cut threshold $\epsilon$.
		\Ensure  $T_{Opt}$-an optimal isolation tree.
        \State  Train LSHiForest to get $T_{LSH}$;   {\color{gray}\Comment{Pre-training}}
        \State Traverse $T_{LSH}$ to get $\Gamma(\epsilon)$;
        \State $\mathbb{C}_0 \leftarrow \mathbb{C}_{\Gamma(\epsilon)}$; $\Gamma_0 \leftarrow \Gamma(\epsilon)$
        {\color{gray}\Comment{Initialisation}}
        \While{$|\mathbb{C}_i|>1$}
        \State Generate a branching factor $v$ from $\mathcal{D}$;
        \If{$|\mathbb{C}_i|~{\leq}~v$}
        \State \Return $T_{Opt}$ with $(N_{\text{root}}, \mathcal{C}^{\prime}) \leftarrow \text{merge}(\mathbb{C}_i)$;
        \EndIf
        \State $J_{\text{cur}}\leftarrow +\infty, \mathbb{C}_{\text{cur}}\leftarrow NULL$
        \ForAll {$\mathbb{C}_v=\{\mathcal{C}_{i1}, \cdots, \mathcal{C}_{iv}\}\subset \mathbb{C}_i$}
        \If{$\text{dist}(\mathbb{C}_v) < J_{\text{cur}}$}
        \State $J_{\text{cur}}\leftarrow \text{dist}(\mathbb{C}_v)$, $\mathbb{C}_{\text{cur}}\leftarrow \mathbb{C}_v$
        \EndIf
        \EndFor 
        \State $(N^{\prime}, \mathcal{C}^{\prime}) \leftarrow \text{merge}(\mathbb{C}_{\text{cur}})$, ;
        {\color{gray}\Comment{Optimal identified}}
        \State $\mathbb{C}_{i+1} \leftarrow \mathbb{C}_{i} \setminus \mathbb{C}_{\text{cur}}$, $\mathbb{C}_{i+1} \leftarrow \mathbb{C}_{i+1} \cup \{\mathcal{C}^{\prime}\}$;
        \State $\Gamma_{i+1}\leftarrow \Gamma_i\setminus \{N_{i1}, \cdots, N_{iv}\}$, $\Gamma_{i+1}\leftarrow \Gamma_{i+1} \cup \{N^{\prime}\}$;
        \EndWhile
	\end{algorithmic} 
\end{algorithm}

Let $T_{LSH}$ denote an isolation tree in LSHiForest. It can be horizontally partitioned into two disjoint parts by a cut which can be represented as a set of node. Let $\Gamma$ denote a cut, and $\Gamma = \{N_1, \cdots, N_{n_0}\}$, where $N_i$, $1~{\leq}~i~{\leq}~n_0$ represents a node, and $n_0$ is the number of the nodes in $\Gamma$. The node $N_i$ is either a subtree of $T_{LSH}$ or a leaf (a trivial subtree), and contains a dataset which can be regarded as a cluster denoted as $\mathcal{C}_i$. Thus, we can obtain a set of clusters $\mathbb{C}_{\Gamma}=\{\mathcal{C}_1, \cdots, \mathcal{C}_{n_0}\}$ from a cut $\Gamma$. Then, we can define the concept of $\epsilon$-cut below.
\begin{definition}[$\epsilon$-Cut]
Given a cut $\Gamma$ and its associated clusters $\mathbb{C}_{\Gamma}=\{\mathcal{C}_1, \cdots, \mathcal{C}_{n_0}\}$, for any node $N_i~{\in}~\Gamma$, if its cluster size satisfies $|\mathcal{C}_i|~{\leq}~\epsilon<|\mathcal{C}_p|$, where $\epsilon$ is a threshold and $\mathcal{C}_p$ is the cluster associated with its parent node, $\Gamma$ is called $\epsilon$-cut and denoted as $\Gamma(\epsilon)$.
\end{definition}

We can construct $\Gamma(\epsilon)$ by simply traversing the tree, and use the associated clusters $\mathbb{C}_{\Gamma(\epsilon)}$ as the initial ones for agglomerative clustering. By tuning $\epsilon~{\in}~[1,\psi]$, where $\psi$ is the sample size in isolation forests, we can control the number and size of the clusters to further, and further manipulate the degree of learning. A higher $\epsilon$ leads to fewer initial clusters and implies less learning. The initialisation degenerates into the conventional case if $\epsilon=1$, while if $\epsilon=\psi$, the resultant isolation forest is still $T_{LSH}$ without learning anything.

Unlike traditional agglomerative hierarchical clustering where a binary tree is generated, we need to achieve multi-fork branches for an optimal isolation tree. The distortion measure is employed to capture the learning loss caused by merging multiple clusters. Let $\boldsymbol{\mu}_{\mathbb{C}}$ denote the merged centre (mean) of a set of clusters $\mathbb{C}=\{\mathcal{C}_1, \cdots, \mathcal{C}_{v}\}$, calculated by

\begin{equation}
    \boldsymbol{\mu}_{\mathbb{C}}= \frac{\sum_{i=1}^v\boldsymbol{\mu}_{\mathcal{C}_i}{\cdot}n_i}{\sum_{i=1}^vn_i},
\end{equation}

{\noindent}where $\boldsymbol{\mu}_{\mathcal{C}}$ is the centre of cluster $\mathcal{C}=\{\boldsymbol{x}_1, \cdots, \boldsymbol{x}_n\}$, calculated by $\boldsymbol{\mu}_{\mathcal{C}} = \frac{1}{n}\sum_{i=1}^{n}\boldsymbol{x}_i$. Then, the distortion of merging the clusters in $\mathbb{C}$ can be calculated as 
\begin{equation}
\text{dist}(\mathbb{C}) =\sum_{i=1}^v || \boldsymbol{\mu}_{\mathcal{C}_i}-\boldsymbol{\mu}_{\mathbb{C}}|| \cdot n_i.
\end{equation}

After generating the branching factor $v$ (Line 5 in Algorithm \ref{alg1}), Line 8-14 shows the details of performing cluster merging. We need to try all the combinations of $v$ clusters and evaluate their distortions. The one with minimum distortion is selected for merging, and a new node is created for the resultant cluster. This process repeats until the number of candidate clusters is not greater than the given branching factor. Finally, all the remaining clusters are merged as the root node of the resultant isolation tree.

It is worth noting the produced tree is not an exact optimal isolation tree if $\epsilon$ is not $1$, because this practical algorithm takes the trade-off between accuracy and efficiency into account. The most time-consuming part is the evaluation of all the combinations of $v$ clusters, whose time complexity grows exponentially with respect to $v$. Thus, to be efficient, our approach only leverages binary and ternary branches for the tree construction, i.e., we adopt a finite distribution, which is technically reasonable because the probability drops considerably when the branching factor increases. But note that other distributions can be used if other clustering techniques like K-means are used.

\begin{table*}[htbp]
    \scriptsize
    
    \setlength\tabcolsep{2.0pt}
    \caption{AUC-ROC and AUC-PR performance (mean $\pm$ standard deviation) of all methods. Our \textbf{OptIForest} method outperforms others.}
	\centering
	\label{tab:auc}
 \resizebox{1.0\textwidth}{!}{%
	\begin{tabular}{l|llllll>{\columncolor[gray]{0.9}}l|llllll>{\columncolor[gray]{0.9}}l} 
		\toprule
         & \multicolumn{7}{c|}{AUC-ROC (\%)} & \multicolumn{7}{c}{AUC-PR (\%)} \\
        \midrule
		Datasets & iForest & LSHiForest & ECOD & REPEN & RDP & DiForest & \textbf{OptIForest} & iForest & LSHiForest & ECOD & REPEN & RDP & DiForest & \textbf{OptIForest} \\ 
		\midrule 
		AD & 69.3 $\pm$ 1.9 & \textbf{77.9 $\pm$ 0.4} & 69.8 $\pm$ 0 & 70.0 $\pm$ 2.2 & \textbf{88.7 $\pm$ 0.3} & 76.8 $\pm$ 0.7 & 77.4 $\pm$ 0.7 & 40.1 $\pm$ 4.9 & 46.2 $\pm$ 0.6 & 48.3 $\pm$ 0 & 37.7 $\pm$ 4.0 & \textbf{72.6 $\pm$ 0.7} & \textbf{51.8 $\pm$ 2.7} & 43.8 $\pm$ 3.1 \\
		campaign & 70.9 $\pm$ 1.0 & 67.7 $\pm$ 0.6 & \textbf{77.5 $\pm$ 0} & 61.1 $\pm$ 4.4 & \textbf{76.3 $\pm$ 0.8} & 69.3 $\pm$ 0.9 & 74.7 $\pm$ 0.4 & 28.5 $\pm$ 1.3 & 24.7 $\pm$ 1.0 & \textbf{35.6 $\pm$ 0} & 17.8 $\pm$ 3.9 & \textbf{37.2 $\pm$ 0.9} & 27.5 $\pm$ 1.3 & 32.0 $\pm$ 0.5 \\
		Arrhythmia & \textbf{79.7 $\pm$ 1.0} & 77.5 $\pm$ 0.5 & \textbf{82.3 $\pm$ 0} & 73.7 $\pm$ 3.6 & 75.5 $\pm$ 0.5 & 76.3 $\pm$ 1.1 & 79.6 $\pm$ 0.8 & \textbf{47.5 $\pm$ 1.4} & 38.6 $\pm$ 0.7 & \textbf{49.4 $\pm$ 0} & 37.4 $\pm$ 3.5 & 32.0 $\pm$ 0.6 & 38.2 $\pm$ 1.2 & 45.1 $\pm$ 1.2 \\
		cardio & \textbf{93.0 $\pm$ 0.7} & 90.4 $\pm$ 0.6 & \textbf{95.0 $\pm$ 0} & 91.5 $\pm$ 2.9 & 88.1 $\pm$ 0.6 & \textbf{93.0 $\pm$ 0.5} & 92.8 $\pm$ 1.3 & 57.0 $\pm$ 2.8 & 48.9 $\pm$ 1.0 & \textbf{67.6 $\pm$ 0} & 53.3 $\pm$ 10.6 & 53.9 $\pm$ 1.5 & 58.7 $\pm$ 1.9 & \textbf{58.9 $\pm$ 3.7} \\
		backdoor & 72.7 $\pm$ 2.9 & 89.2 $\pm$ 0.9 & 84.9 $\pm$ 0 & 86.8 $\pm$ 1.6 & 91.0 $\pm$ 2.1 & \textbf{92.0 $\pm$ 0.5} & \textbf{92.7 $\pm$ 0.5} & 4.5 $\pm$ 0.7 & \textbf{27.3 $\pm$ 2.6} & 9.6 $\pm$ 0 & 12.5 $\pm$ 1.7 & 3.5 $\pm$ 0.8 & 39.4 $\pm$ 3.3 & \textbf{51.7 $\pm$ 8.7} \\
		KDDCup99 & \textbf{97.0 $\pm$ 0.6} & 96.4 $\pm$ 0.2 & 91.1 $\pm$ 0 & 95.9 $\pm$ 0.6 & 41.0 $\pm$ 3.1 & 88.5 $\pm$ 0.7 & \textbf{97.4 $\pm$ 0.1} & \textbf{48.6 $\pm$ 7.0} & 32.6 $\pm$ 1.1 & \textbf{48.5 $\pm$ 0} & 44.1 $\pm$ 1.9 & 15.4 $\pm$ 0.9 & 16.7 $\pm$ 0.5 & 43.0 $\pm$ 0.1 \\
		Celeba & 69.4 $\pm$ 2.4 & 72.5 $\pm$ 0.7 & 72.3 $\pm$ 0 & \textbf{84.3 $\pm$ 2.2} & \textbf{86.0 $\pm$ 0.6} & 67.9 $\pm$ 1.7 & 79.2 $\pm$ 1.9 & 6.3 $\pm$ 0.9 & 6.8 $\pm$ 0.3 & 8.5 $\pm$ 0 & \textbf{10.7 $\pm$ 1.8} & \textbf{10.4 $\pm$ 0.6} & 5.5 $\pm$ 0.6 & 8.1 $\pm$ 1.0 \\
		mnist & 80.2 $\pm$ 1.8 & \textbf{85.3 $\pm$ 0.6} & 83.8 $\pm$ 0 & 67.6 $\pm$ 10.6 & 85.1 $\pm$ 1.6 & 83.7 $\pm$ 1.3 & \textbf{85.5 $\pm$ 1.0} & 27.7 $\pm$ 3.2 & \textbf{38.3 $\pm$ 1.0} & 30.5 $\pm$ 0 & 20.4 $\pm$ 10.2 & 36.7 $\pm$ 2.4 & 33.0 $\pm$ 2.1 & \textbf{40.7 $\pm$ 1.9} \\
		Census & 60.1 $\pm$ 1.8 & 62.6 $\pm$ 0.4 & \textbf{66.8 $\pm$ 0} & 62.7 $\pm$ 1.5 & 65.3 $\pm$ 0.4 & 59.4 $\pm$ 1.1 & \textbf{67.8 $\pm$ 1.0} & 7.1 $\pm$ 0.3 & 7.5 $\pm$ 0.1 & \textbf{8.6 $\pm$ 0} & 7.7 $\pm$ 0.2 & \textbf{8.6 $\pm$ 0.1} & 6.9 $\pm$ 0.2 & \textbf{8.8 $\pm$ 0.2} \\
		Donors & 76.6 $\pm$ 1.0 & 74.5 $\pm$ 0.7 & 74.0 $\pm$ 0 & \textbf{83.2 $\pm$ 1.7} & \textbf{96.2 $\pm$ 1.1} & 67.8 $\pm$ 1.2 & 77.1 $\pm$ 3.2 & 11.9 $\pm$ 0.8 & 10.5 $\pm$ 0.6 & 13.6 $\pm$ 0 & \textbf{15.5 $\pm$ 1.3} & \textbf{43.2 $\pm$ 6.1} & 8.0 $\pm$ 0.3 & 11.3 $\pm$ 1.6 \\
		Cover & 88.0 $\pm$ 2.1 & \textbf{93.6 $\pm$ 0.5} & \textbf{93.3 $\pm$ 0} & 86.6 $\pm$ 5.7 & 51.2 $\pm$ 1.3 & 76.4 $\pm$ 4.0 & 90.9 $\pm$ 0.8 & 6.4 $\pm$ 0.9 & \textbf{9.0 $\pm$ 0.8} & \textbf{11.6 $\pm$ 0} & 5.3 $\pm$ 2.0 & 2.0 $\pm$ 1.1 & 4.0 $\pm$ 0.8 & 6.5 $\pm$ 0.7 \\
		http & \textbf{99.9 $\pm$ 0} & 93.3 $\pm$ 0 & 97.9 $\pm$ 0 & \textbf{99.4 $\pm$ 0.1} & 99.3 $\pm$ 0.1 & 99.3 $\pm$ 0.1 & \textbf{99.4 $\pm$ 0.1} & \textbf{90.2 $\pm$ 7.9} & 34.2 $\pm$ 0.6 & 14.5 $\pm$ 0 & \textbf{39.5 $\pm$ 2.4} & 36.2 $\pm$ 0.8 & 35.1 $\pm$ 0.4 & 35.4 $\pm$ 1.5 \\
		smtp & 90.5 $\pm$ 0.8 & 86.9 $\pm$ 0.9 & 88.0 $\pm$ 0 & \textbf{90.9 $\pm$ 1.3} & 69.8 $\pm$ 1.1 & 84.6 $\pm$ 0.5 & \textbf{92.4 $\pm$ 0.6} & 0.4 $\pm$ 0 & \textbf{56.9 $\pm$ 3.1} & 50.7 $\pm$ 0 & 26.9 $\pm$ 10.8 & 21.7 $\pm$ 3.9 & \textbf{55.2 $\pm$ 7.0} & 36.9 $\pm$ 13.7 \\
		Ionosphere & 84.5 $\pm$ 0.5 & 91.2 $\pm$ 0.2 & 76.8 $\pm$ 0 & 86.5 $\pm$ 2.3 & 82.7 $\pm$ 0.6 & \textbf{94.3 $\pm$ 0.5} & \textbf{93.4 $\pm$ 0.3} & 80.8 $\pm$ 0.5 & 89.6 $\pm$ 0.4 & 66.3 $\pm$ 0 & 81.0 $\pm$ 3.6 & 80.7 $\pm$ 0.6 & \textbf{93.0 $\pm$ 0.6} & \textbf{92.3 $\pm$ 0.5} \\
		Satellite & 70.3 $\pm$ 1.8 & \textbf{76.7 $\pm$ 0.5} & 74.6 $\pm$ 0 & 74.0 $\pm$ 3.0 & 60.9 $\pm$ 0.6 & 69.2 $\pm$ 1.3 & \textbf{78.6 $\pm$ 0.7} & 65.0 $\pm$ 2.1 & 63.5 $\pm$ 0.6 & 66.1 $\pm$ 0 & \textbf{70.3 $\pm$ 3.9} & 61.3 $\pm$ 0.7 & 47.5 $\pm$ 1.3 & \textbf{71.5 $\pm$ 0.9} \\
		Shuttle & \textbf{99.7 $\pm$ 0.1} & 97.2 $\pm$ 0.8 & \textbf{99.7 $\pm$ 0} & 99.4 $\pm$ 0.1 & 99.6 $\pm$ 0.9 & 96.3 $\pm$ 1.2 & 98.0 $\pm$ 0.4 & \textbf{97.6 $\pm$ 0.5} & 40.1 $\pm$ 2.7 & \textbf{95.0 $\pm$ 0} & 91.9 $\pm$ 0.6 & 89.7 $\pm$ 1.1 & 56.5 $\pm$ 6.5 & 64.0 $\pm$ 4.3 \\
		Spam & 61.7 $\pm$ 2.9 & 70.7 $\pm$ 0.2 & 65.6 $\pm$ 0 & \textbf{73.4 $\pm$ 0.3} & \textbf{74.7 $\pm$ 0.5} & 65.0 $\pm$ 0.9 & 71.1 $\pm$ 0.2 & 46.8 $\pm$ 3.0 & 59.1 $\pm$ 0.5 & 51.8 $\pm$ 0 & \textbf{62.7 $\pm$ 0.5} & \textbf{63.0 $\pm$ 0.2} & 59.6 $\pm$ 0.5 & 58.1 $\pm$ 0.5 \\
		Vowel & 75.3 $\pm$ 1.2 & \textbf{90.8 $\pm$ 0.7} & 40.8 $\pm$ 0 & 77.9 $\pm$ 5.5 & 67.2 $\pm$ 1.2 & \textbf{94.4 $\pm$ 1.4} & 90.0 $\pm$ 1.2 & 12.6 $\pm$ 1.7 & 29.5 $\pm$ 2.1 & 2.8 $\pm$ 0 & 14.5 $\pm$ 6.5 & 7.6 $\pm$ 3.2 & \textbf{41.1 $\pm$ 9.0} & \textbf{32.4 $\pm$ 4.3} \\
		Waveform & 69.6 $\pm$ 2.3 & 70.7 $\pm$ 0.8 & \textbf{71.5 $\pm$ 0} & 63.2 $\pm$ 9.5 & 66.9 $\pm$ 0.9 & 63.2 $\pm$ 3.4 & \textbf{74.4 $\pm$ 1.3} & 9.7 $\pm$ 0.9 & 9.8 $\pm$ 0.4 & 9.4 $\pm$ 0 & 7.6 $\pm$ 2.4 & 8.7 $\pm$ 1.5 & \textbf{11.2 $\pm$ 3.3} & \textbf{11.3 $\pm$ 1.0} \\
		Wine & 63.9 $\pm$ 0.7 & \textbf{67.0 $\pm$ 0.3} & 62.5 $\pm$ 0 & 64.6 $\pm$ 2.2 & 61.7 $\pm$ 0.4 & 65.9 $\pm$ 1.2 & \textbf{66.4 $\pm$ 0.5} & 7.8 $\pm$ 0.4 & \textbf{8.9 $\pm$ 0.2} & 8.1 $\pm$ 0 & 8.1 $\pm$ 0.7 & 8.1 $\pm$ 0.2 & \textbf{8.5 $\pm$ 0.5} & 8.1 $\pm$ 0.2 \\
		\midrule
		Average & 74.3 $\pm$ 1.4 & 81.9 $\pm$ 0.5 & 78.8 $\pm$ 0 & 79.6 $\pm$ 3.1 & 81.2 $\pm$ 0.8 & 79.2 $\pm$ 1.2 & \textbf{83.9 $\pm$ 0.9} & 34.8 $\pm$ 2.5 & 34.1 $\pm$ 1.0 & 34.8 $\pm$ 0 & 33.3 $\pm$ 3.6 & 37.4 $\pm$ 1.4 & 34.9 $\pm$ 2.2 & \textbf{38.0 $\pm$ 3.7} \\
		\bottomrule
	\end{tabular}
 }
\end{table*}

Once a forest of optimal isolation trees are built as an optimal isolation forest for OptIForest, it can perform anomaly detection with the same way to derive anomaly scores as LSHiForest. Although a practical isolation tree can consists of two parts with the upper layers resulting from clustering and the lower layers from LSHiForest, the hash function interface makes the two implementations (LSH functions and the non-parametric learning to hash function in Eq.(7)) no difference as to anomaly score computation.




\section{Experiments}

\subsection{Experiment Setting}
\noindent \textbf{Baselines.} Our method OptIForest is compared with six state-of-the-art anomaly detection methods: iForest, LSHiForest, ECOD, REPEN, RDP, and DiForest. We evaluate all methods on 20 widely-used benchmark datasets~\cite{pang2019deep,han2022adbench,li2022ecod}. We refer the reader to Appendix \ref{appendix:baseline} for more details of the baselines and datasets.

\noindent \textbf{Metrics.} We conventionally use the Area Under Receiver Operating Characteristic Curve (AUC-ROC) and Area Under the Precision-Recall Curve (AUC-PR) as the performance metrics \cite{kurt2020real,wang2021unsupervised}. To ensure a fair comparison, we follow the optimal parameter settings of the baseline methods. Please see Appendix \ref{appendix:setup} for more details about the metrics and the parameter settings. All experiments are run 15 times and averaged results are reported.

\subsection{Comparison Study Results and Discussion}

\noindent \textbf{AUC-ROC Results.}~~Table~\ref{tab:auc} illustrates that OptIForest is the most robust method that achieves the best performance on most datasets. Specifically, OptIForest has the highest average AUC-ROC score among all the compared methods, surpassing the second best method by 2\%. Additionally, RDP and LSHiForest perform well and are stable on most datasets. ECOD produces exceptional results on certain datasets, but performs poorly on others like ``Vowel'' and ``CD'' due to its requirement on the data distribution. Note that ECOD has zero standard deviation because it is a deterministic method. The rest methods iForest, REPEN, and DiForest only perform well on a few datasets, indicating their lack of robustness across all datasets. Thus, it can be seen that no learning in iForest or DiForest fail to achieve a robust performance due to the lack of knowledge learned from data. LSHiForest and OptIForest with better tree structures can achieve bias reduction, and learning from data further makes OptIForest perform better and even outperform the deep detector.

\noindent \textbf{AUC-PR Results.}~~The results in Table~\ref{tab:auc} show that OptIForest consistently performs best across various datasets with the highest average AUC-PR score, showing the superiority of our method. Specifically, our method outperforms the second best one by about 0.6\%. The top two results of each dataset are highlighted in bold. Although iForest and ECOD have exceptional results on a few datasets, they perform poorly on others. REPEN has the lowest average AUC-PR score and performs poorly on most datasets. The AUC-PR results show the same trend as the AUC-ROC results, supporting that our OptIForest performs the best performance in a robust way.

\noindent \textbf{Execution Time.}~~We use execution time as the efficiency metric to compare OptIForest with other methods, and the results are reported in Table $3$ in Appendix \ref{execution time}. OptIForest has much shorter execution time than the deep learning methods of REPEN and RDP for most datasets. As expected, OptIForest reasonably has longer execution time than other isolation forests where no learning is conducted. So, it can be seen that our approach strikes a good balance between execution efficiency and detection performance. It is appealing that OptIForest can achieve better performance than deep learning methods but takes much less execution cost.

\subsection{Ablation Study Results and Discussion}
\label{sect:Ablation}
To understand how the branching factor, the cut threshold, and the sampling size influence the performance of OptIForest, we performed detailed ablation studies on eight datasets with a range of data types and sizes.


\paragraph{Branching Factor.}~~To eliminate the effects of feature learning, we use data-independent baseline (i.e., $\epsilon=\psi$) to analyze the impact of the branching factor. Besides, it is difficult to implement the average branch of $e$ in practical experiment because many of the underlying branches just produce 2 branch forks. However, we can still analyze the AUC-ROC results for branches that are either near or far from $e$. The AUC-ROC results for different branching factors are presented in Fig.~\ref{figbranch}. When the branching factor is close to $e$, the AUC-ROC results are the best on almost all datasets. This result means that the best anomaly detection accuracy can be achieved if the branching factor satisfies the condition: $v \approx e$. These experimental results validate Theorem~\ref{theorem} in Section~\ref{optimaltree}.
\begin{figure}[!tbp]
  \centering  
  \includegraphics[width=0.4\textwidth]{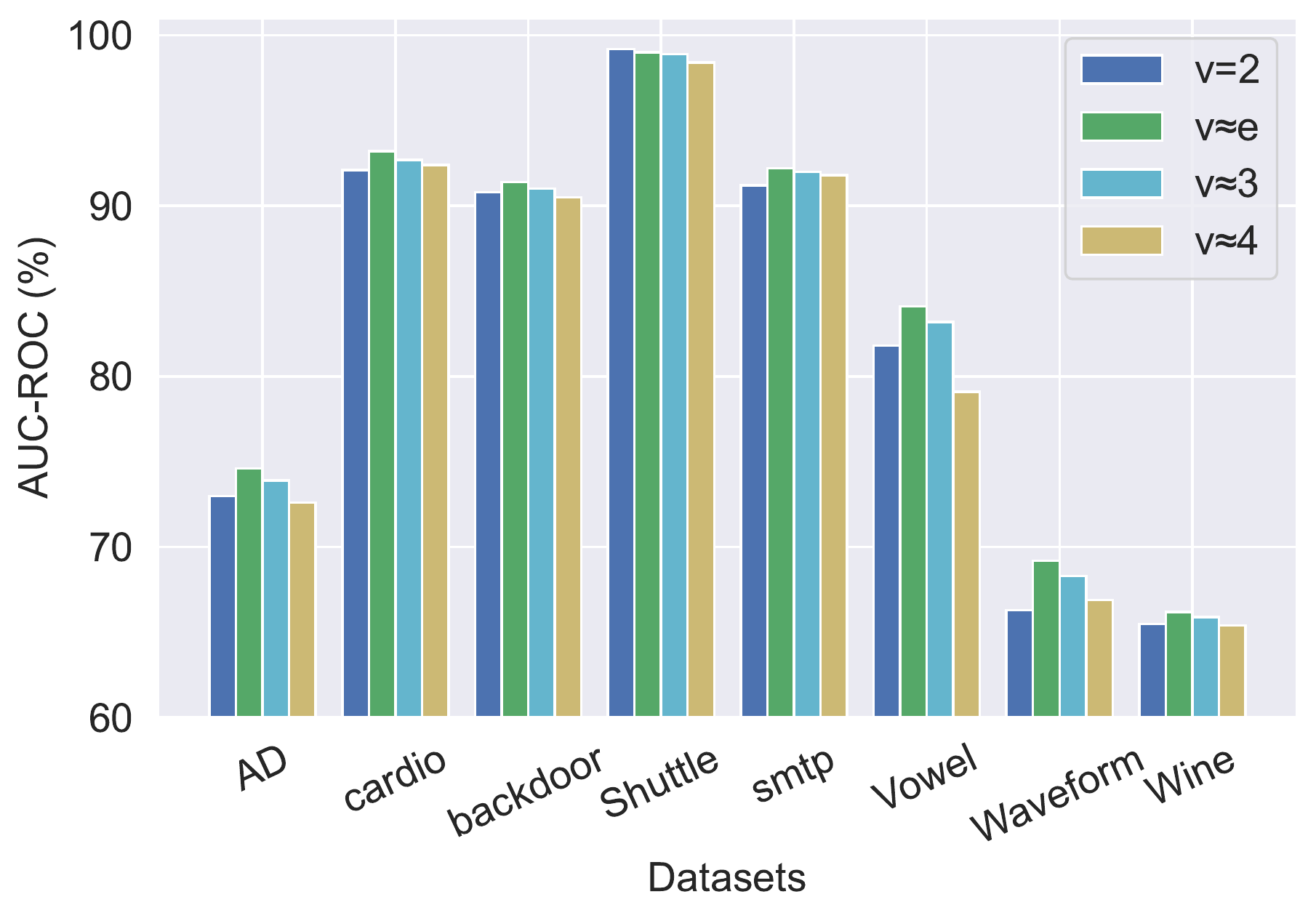}
  \captionsetup{font={small}}
  \caption{Detection performance changes w.r.t. branching factor $v$. A branching factor closer to $e$ leads to a better performance.}
  \label{figbranch}
\end{figure}

\paragraph{Cut Threshold $\epsilon$.}~~The cut threshold is studied by raising the branching factor to the power of $e$, as small changes of the threshold do not greatly affect detection accuracy. We also study the boundary condition of the cut threshold ($\epsilon=512$). Fig.~\ref{figthreshold} displays AUC-ROC results and standard deviations for different $\epsilon$. It can be seen that the curve increases as the threshold increases in four large datasets with data sizes larger than 10,000 or dimension sizes larger than 1,000, indicating little learning is required for large datasets (with $\epsilon=e^6$ as a reference). Conversely, more learning is necessary in four small datasets, and $\epsilon=e^4$ serves as a good reference point to balance accuracy and time efficiency. A comparison between the results of $\epsilon=512$ with others illustrates that appropriate learning results in better outcomes than not learning on most datasets. But it is exceptional in the ``AD'' and ``vowel'' datasets, where not learning yields better results. This could be attributed to the fact that the isolation forest without learning exactly has the average branch of $e$, which achieves a favorable bias-variance trade-off.
\begin{figure}[!t]
  \centering  
  \includegraphics[width=0.4\textwidth]{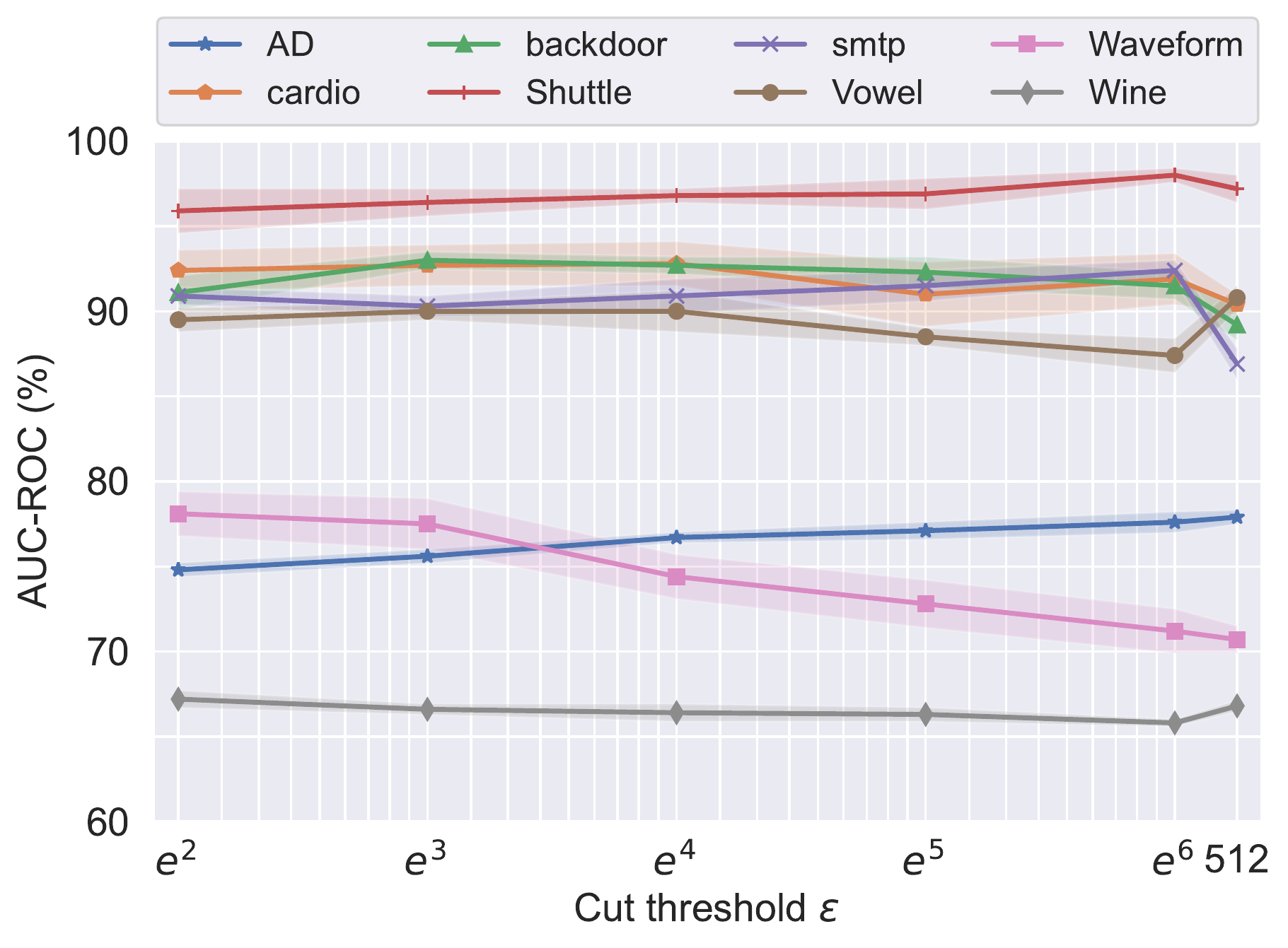}
  \captionsetup{font={small}}
  \caption{Detection performance changes w.r.t. cut threshold $\epsilon$. In general, a certain level of learning outperforms the case of no learning ($\epsilon=512$ means no learning).}
  \label{figthreshold}
\end{figure}

\paragraph{Sampling Size.}~~To facilitate the observation of the results, the sampling size is set from $2^6$ to $2^{11}$ with exponential increase, as the results in this range can display significant change. The cut threshold of the same sampling size will impact the AUC-ROC results, so we determine the best results for each sampling group as the final outcomes. It can be concluded from Fig.~\ref{figsample} that the AUC-ROC results improve with the increase of the sampling size on most datasets. The AUC-ROC results remain stable once the sampling size exceed $2^9$, except for the ``cardio'' and ``Shuttle'' datasets, which keep stable on any sampling size. 

\begin{figure}[!t]
  \centering  
  \includegraphics[width=0.4\textwidth]{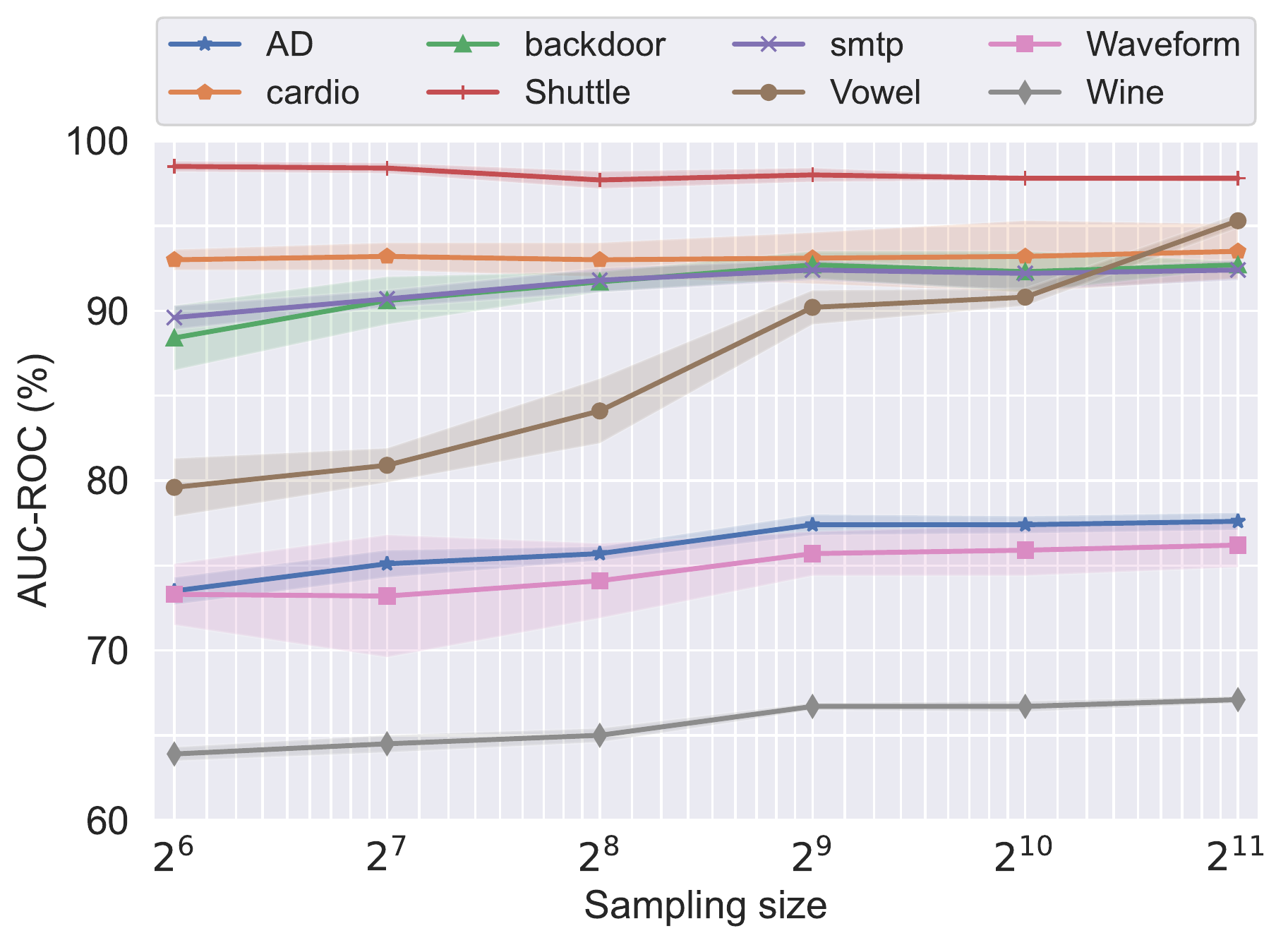}
  \captionsetup{font={small}}
  \caption{Detection performance changes w.r.t. sampling size $\psi$. In general, the performance is saturated when $\psi$ reaches a threshold ($2^9=512$ herein).}
  \label{figsample}
\end{figure}

\section{Conclusion and Future Work}
Isolation Forest is an appealing anomaly detection method given its salient characteristics, but it lacks a theoretical foundation about the structure optimality of an isolation tree. In this paper, we have introduced the concept of isolation efficiency and formulated a constrained optimisation problem to derive the optimal branching factor for an isolation tree. We have shown that an optimal isolation tree or forest is theoretically with the branching factor $e$. Furthermore, we have developed a practical optimal isolation forest OptIForest which can achieve both high computational efficiency and a good bias-variance trade-off by designing a novel clustering based learning to hash for data isolation. We have conducted extensive experiments on a variety of benchmarking datasets for both ablation and comparative studies, and the results have confirmed the effectiveness and efficiency of OptIForest. In the future, we plan to design a version in the context of federated learning where data are scattered across multiple clients.

\clearpage




\bibliographystyle{named}
\bibliography{ijcai22}

\begin{thebibliography}{}

\bibitem[\protect\citeauthoryear{Aggarwal and
  Sathe}{2015}]{aggarwal2015theoretical}
Charu~C Aggarwal and Saket Sathe.
\newblock Theoretical foundations and algorithms for outlier ensembles.
\newblock {\em Acm sigkdd explorations newsletter}, 17(1):24--47, 2015.

\bibitem[\protect\citeauthoryear{Ahmed \bgroup \em et al.\egroup
  }{2016}]{ahmed2016survey}
Mohiuddin Ahmed, Abdun~Naser Mahmood, and Jiankun Hu.
\newblock A survey of network anomaly detection techniques.
\newblock {\em Journal of Network and Computer Applications}, 60:19--31, 2016.

\bibitem[\protect\citeauthoryear{Bandaragoda \bgroup \em et al.\egroup
  }{2014}]{bandaragoda2014efficient}
Tharindu~R Bandaragoda, Kai~Ming Ting, David Albrecht, Fei~Tony Liu, and
  Jonathan~R Wells.
\newblock Efficient anomaly detection by isolation using nearest neighbour
  ensemble.
\newblock In {\em IEEE International conference on data mining workshop (ICDM
  workshop)}, pages 698--705. IEEE, 2014.

\bibitem[\protect\citeauthoryear{Bawa \bgroup \em et al.\egroup
  }{2005}]{bawa2005lsh}
Mayank Bawa, Tyson Condie, and Prasanna Ganesan.
\newblock Lsh forest: self-tuning indexes for similarity search.
\newblock In {\em Proceedings of the 14th international conference on World
  Wide Web (WWW)}, pages 651--660, 2005.

\bibitem[\protect\citeauthoryear{Chandola \bgroup \em et al.\egroup
  }{2009}]{chandola2009anomaly}
Varun Chandola, Arindam Banerjee, and Vipin Kumar.
\newblock Anomaly detection: A survey.
\newblock {\em ACM computing surveys}, 41(3):1--58, 2009.

\bibitem[\protect\citeauthoryear{Chen and
  Tsourakakis}{2022}]{chen2022antibenford}
Tianyi Chen and Charalampos Tsourakakis.
\newblock Antibenford subgraphs: Unsupervised anomaly detection in financial
  networks.
\newblock In {\em Proceedings of the 28th ACM SIGKDD Conference on Knowledge
  Discovery and Data Mining (SIGKDD)}, pages 2762--2770, 2022.

\bibitem[\protect\citeauthoryear{Chen \bgroup \em et al.\egroup
  }{2017}]{chen2017outlier}
Jinghui Chen, Saket Sathe, Charu Aggarwal, and Deepak Turaga.
\newblock Outlier detection with autoencoder ensembles.
\newblock In {\em Proceedings of the 2017 SIAM international conference on data
  mining (SDM)}, pages 90--98. SIAM, 2017.

\bibitem[\protect\citeauthoryear{Fernando \bgroup \em et al.\egroup
  }{2021}]{fernando2021deep}
Tharindu Fernando, Harshala Gammulle, Simon Denman, Sridha Sridharan, and
  Clinton Fookes.
\newblock Deep learning for medical anomaly detection--a survey.
\newblock {\em ACM Computing Surveys}, 54(7):1--37, 2021.

\bibitem[\protect\citeauthoryear{Han \bgroup \em et al.\egroup
  }{2022}]{han2022adbench}
Songqiao Han, Xiyang Hu, Hailiang Huang, Minqi Jiang, and Yue Zhao.
\newblock Adbench: Anomaly detection benchmark.
\newblock {\em Advances in Neural Information Processing Systems (NeurIPS)},
  2022.

\bibitem[\protect\citeauthoryear{Hariri \bgroup \em et al.\egroup
  }{2019}]{hariri2019extended}
Sahand Hariri, Matias~Carrasco Kind, and Robert~J Brunner.
\newblock Extended isolation forest.
\newblock {\em IEEE Transactions on Knowledge and Data Engineering},
  33(4):1479--1489, 2019.

\bibitem[\protect\citeauthoryear{Kurt \bgroup \em et al.\egroup
  }{2020}]{kurt2020real}
Mehmet~Necip Kurt, Yasin Yilmaz, and Xiaodong Wang.
\newblock Real-time nonparametric anomaly detection in high-dimensional
  settings.
\newblock {\em IEEE Transactions on Pattern Analysis and Machine Intelligence},
  2020.

\bibitem[\protect\citeauthoryear{Li \bgroup \em et al.\egroup
  }{2022}]{li2022ecod}
Zheng Li, Yue Zhao, Xiyang Hu, Nicola Botta, Cezar Ionescu, and George Chen.
\newblock Ecod: Unsupervised outlier detection using empirical cumulative
  distribution functions.
\newblock {\em IEEE Transactions on Knowledge and Data Engineering}, 2022.

\bibitem[\protect\citeauthoryear{Liu \bgroup \em et al.\egroup
  }{2008}]{liu2008isolation}
Fei~Tony Liu, Kai~Ming Ting, and Zhi-Hua Zhou.
\newblock Isolation forest.
\newblock In {\em IEEE international conference on data mining (ICDM)}, pages
  413--422, 2008.

\bibitem[\protect\citeauthoryear{Liu \bgroup \em et al.\egroup
  }{2010}]{liu2010detecting}
Fei~Tony Liu, Kai~Ming Ting, and Zhi~Hua Zhou.
\newblock On detecting clustered anomalies using sciforest.
\newblock In {\em Joint European Conference on Machine Learning and Knowledge
  Discovery in Databases (ECML-PKDD)}, pages 274--290. Springer, 2010.

\bibitem[\protect\citeauthoryear{Liu \bgroup \em et al.\egroup
  }{2012}]{liu2012isolation}
Fei~Tony Liu, Kai~Ming Ting, and Zhi-Hua Zhou.
\newblock Isolation-based anomaly detection.
\newblock {\em ACM Transactions on Knowledge Discovery from Data}, 6(1):1--39,
  2012.

\bibitem[\protect\citeauthoryear{Liu \bgroup \em et al.\egroup
  }{2019}]{liu2019generative}
Yezheng Liu, Zhe Li, Chong Zhou, Yuanchun Jiang, Jianshan Sun, Meng Wang, and
  Xiangnan He.
\newblock Generative adversarial active learning for unsupervised outlier
  detection.
\newblock {\em IEEE Transactions on Knowledge and Data Engineering},
  32(8):1517--1528, 2019.

\bibitem[\protect\citeauthoryear{Pang \bgroup \em et al.\egroup
  }{2018}]{pang2018learning}
Guansong Pang, Longbing Cao, Ling Chen, and Huan Liu.
\newblock Learning representations of ultrahigh-dimensional data for random
  distance-based outlier detection.
\newblock In {\em Proceedings of the 24th ACM SIGKDD international conference
  on knowledge discovery \& data mining (SIGKDD)}, pages 2041--2050, 2018.

\bibitem[\protect\citeauthoryear{Pang \bgroup \em et al.\egroup
  }{2019}]{pang2019deep}
Guansong Pang, Chunhua Shen, and Anton van~den Hengel.
\newblock Deep anomaly detection with deviation networks.
\newblock In {\em Proceedings of the 25th ACM SIGKDD international conference
  on knowledge discovery \& data mining (SIGKDD)}, pages 353--362, 2019.

\bibitem[\protect\citeauthoryear{Pang \bgroup \em et al.\egroup
  }{2021a}]{pang2021deep}
Guansong Pang, Chunhua Shen, Longbing Cao, and Anton Van~Den Hengel.
\newblock Deep learning for anomaly detection: A review.
\newblock {\em ACM Computing Surveys}, 54(2):1--38, 2021.

\bibitem[\protect\citeauthoryear{Pang \bgroup \em et al.\egroup
  }{2021b}]{pang2021toward}
Guansong Pang, Anton van~den Hengel, Chunhua Shen, and Longbing Cao.
\newblock Toward deep supervised anomaly detection: Reinforcement learning from
  partially labeled anomaly data.
\newblock In {\em Proceedings of the 27th ACM SIGKDD conference on knowledge
  discovery \& data mining (SIGKDD)}, pages 1298--1308, 2021.

\bibitem[\protect\citeauthoryear{Ruff \bgroup \em et al.\egroup
  }{2021}]{ruff2021unifying}
Lukas Ruff, Jacob~R Kauffmann, Robert~A Vandermeulen, Gr{\'e}goire Montavon,
  Wojciech Samek, Marius Kloft, Thomas~G Dietterich, and Klaus-Robert
  M{\"u}ller.
\newblock A unifying review of deep and shallow anomaly detection.
\newblock {\em Proceedings of the IEEE}, 109(5):756--795, 2021.

\bibitem[\protect\citeauthoryear{Russell}{1991}]{russell1991estimating}
KG~Russell.
\newblock Estimating the value of e by simulation.
\newblock {\em The American Statistician}, 45(1):66--68, 1991.

\bibitem[\protect\citeauthoryear{Wang \bgroup \em et al.\egroup
  }{2017}]{wang2017survey}
Jingdong Wang, Ting Zhang, Nicu Sebe, Heng~Tao Shen, et~al.
\newblock A survey on learning to hash.
\newblock {\em IEEE transactions on pattern analysis and machine intelligence},
  40(4):769--790, 2017.

\bibitem[\protect\citeauthoryear{Wang \bgroup \em et al.\egroup
  }{2021}]{wang2021unsupervised}
Hu~Wang, Guansong Pang, Chunhua Shen, and Congbo Ma.
\newblock Unsupervised representation learning by predicting random distances.
\newblock In {\em International Joint Conference on Artificial Intelligence
  (IJCAI)}, pages 2950--2956, 2021.

\bibitem[\protect\citeauthoryear{Xu \bgroup \em et al.\egroup
  }{2022}]{xu2022deep}
Hongzuo Xu, Guansong Pang, Yijie Wang, and Yongjun Wang.
\newblock Deep isolation forest for anomaly detection.
\newblock {\em arXiv preprint arXiv:2206.06602}, 2022.

\bibitem[\protect\citeauthoryear{Zavrtanik \bgroup \em et al.\egroup
  }{2021}]{zavrtanik2021reconstruction}
Vitjan Zavrtanik, Matej Kristan, and Danijel Sko{\v{c}}aj.
\newblock Reconstruction by inpainting for visual anomaly detection.
\newblock {\em Pattern Recognition}, 112:107706, 2021.

\bibitem[\protect\citeauthoryear{Zha \bgroup \em et al.\egroup
  }{2020}]{zha2020meta}
Daochen Zha, Kwei-Herng Lai, Mingyang Wan, and Xia Hu.
\newblock Meta-aad: Active anomaly detection with deep reinforcement learning.
\newblock In {\em IEEE International conference on data mining (ICDM)}, pages
  771--780. IEEE, 2020.

\bibitem[\protect\citeauthoryear{Zhang \bgroup \em et al.\egroup
  }{2017}]{zhang2017lshiforest}
Xuyun Zhang, Wanchun Dou, Qiang He, Rui Zhou, Christopher Leckie, Ramamohanarao
  Kotagiri, and Zoran Salcic.
\newblock Lshiforest: A generic framework for fast tree isolation based
  ensemble anomaly analysis.
\newblock In {\em IEEE international conference on data engineering (ICDE)},
  pages 983--994, 2017.

\bibitem[\protect\citeauthoryear{Zhang \bgroup \em et al.\egroup
  }{2021}]{zhang2021elite}
Huayi Zhang, Lei Cao, Peter VanNostrand, Samuel Madden, and Elke~A
  Rundensteiner.
\newblock Elite: Robust deep anomaly detection with meta gradient.
\newblock In {\em Proceedings of the 27th ACM SIGKDD Conference on Knowledge
  Discovery \& Data Mining (SIGKDD)}, pages 2174--2182, 2021.

\bibitem[\protect\citeauthoryear{Zimek \bgroup \em et al.\egroup
  }{2013}]{zimek2013subsampling}
Arthur Zimek, Matthew Gaudet, Ricardo~JGB Campello, and J{\"o}rg Sander.
\newblock Subsampling for efficient and effective unsupervised outlier
  detection ensembles.
\newblock In {\em Proceedings of the 19th ACM SIGKDD international conference
  on Knowledge discovery and data mining (SIGKDD)}, pages 428--436, 2013.

\bibitem[\protect\citeauthoryear{Zong \bgroup \em et al.\egroup
  }{2018}]{zong2018deep}
Bo~Zong, Qi~Song, Martin~Renqiang Min, Wei Cheng, Cristian Lumezanu, Daeki Cho,
  and Haifeng Chen.
\newblock Deep autoencoding gaussian mixture model for unsupervised anomaly
  detection.
\newblock In {\em International conference on learning representations (ICLR)},
  2018.

\end{thebibliography}

\clearpage

\appendix
\setcounter{theorem}{0}
\setcounter{equation}{0}
\renewcommand{\theequation}{A.\arabic{equation}}
\setcounter{figure}{0}
\renewcommand{\thefigure}{A.\arabic{figure}}

\counterwithin{proposition}{section}
 
\section{Appendix A}
\label{proof}

\subsection{Proof of Theorem 1}
\label{proof1}

\begin{theorem}
\label{theorem} 
An isolation tree $T$ has the highest isolation efficiency when its branching factor $v=e$, where $e$ is the Euler's number with numerical values around $2.718$.
\end{theorem}

\begin{proof}
According to the Definition of the isolation efficiency, to obtain the optimal isolation tree is to maximise the isolation efficiency, which is formulated as:
\begin{equation}
    \eta(v,d)=\frac{\psi}{\phi}=\frac{v^d}{vd},
\label{formul1}
\end{equation}
where $v$ represents the branching factor, $d$ represents the tree depth. After fixing the isolation area by a constraint number $\Phi$, we can obtain the function of isolation efficiency with respect to $v$:
\begin{equation}
    \eta(v)=\frac{1}{\Phi}v^{\frac{\Phi}{v}}.
\label{formul2}
\end{equation}
The derivative of $\eta(v)$ can be derived by:
\begin{equation}
\eta'(v)=v^{(\frac{\Phi}{v}-2)}(1-\ln v).
\label{formul3}
\end{equation}
Because $v^{(\frac{\Phi}{v}-2)} > 0$, if $1-\ln v > 0$, we can have $\eta'(v) > 0$, and if $1-\ln v < 0$, we can have $\eta'(v) < 0$. Thus, $\eta(v)$ is a convex function and has a maximum when $1-\ln v = 0$, i.e., the optimal branching faction $v^*=e$.

Also, we can visualize the relationship between $v$ and $\eta(v)$. Without loss of generality, Fig.~\ref{fig:function} demonstrates an example where $\Phi$ is fixed at 6 (other values should show the same trend of $\eta(v)$ with respect to $v$). It can be seen that the highest isolation efficiency is attained when the branch factor is equal to $e$.
\begin{figure}[htbp]
  \centering  
  \includegraphics[width=0.45\textwidth]{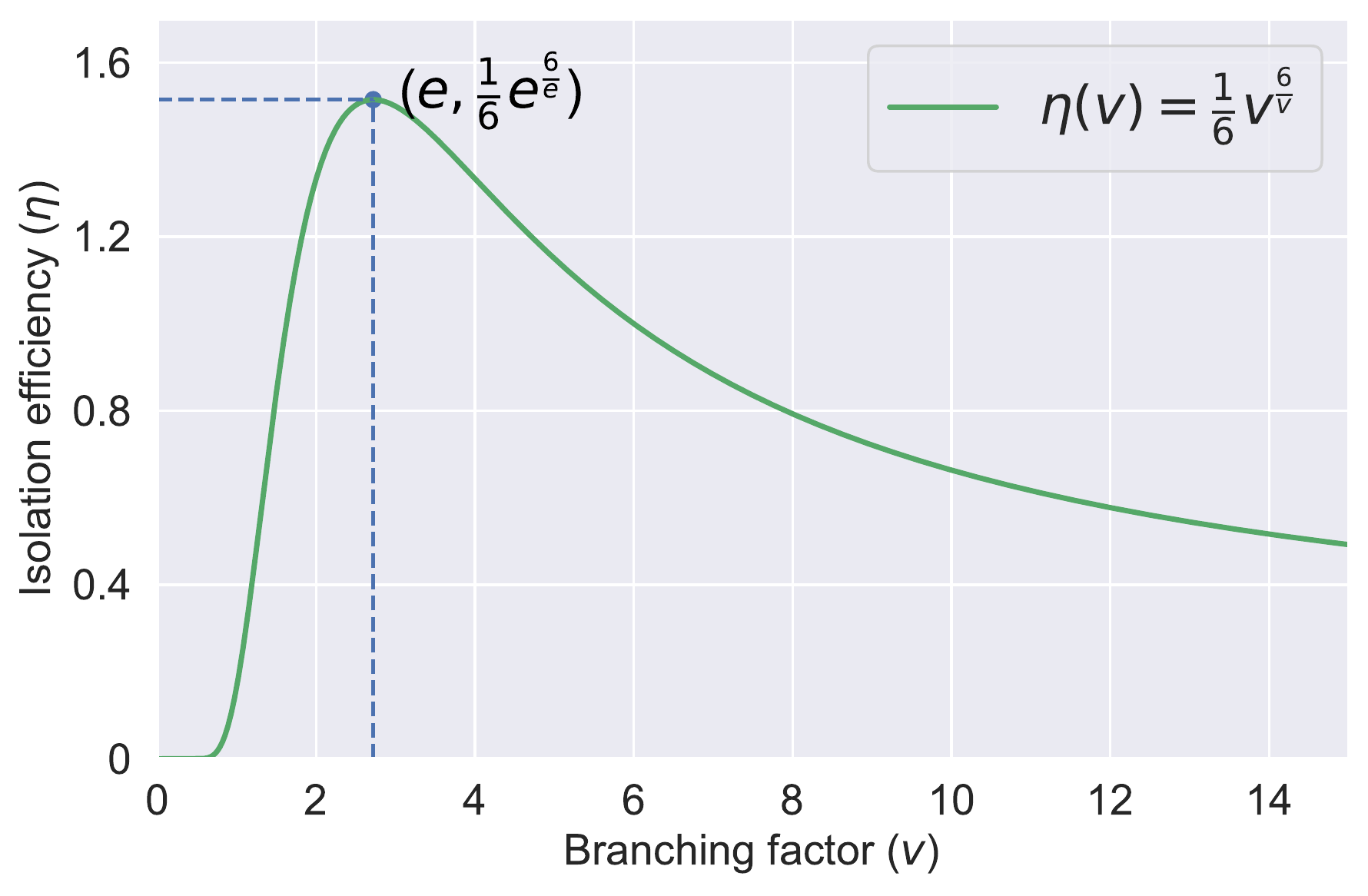}
  \captionsetup{font={small}}
  \caption{The relationship between isolation efficiency $\eta(v)$ and branching factor $v$. We use $\Phi=6$ as an example for an illustration without loss of generality.}
  \label{fig:function}
\end{figure}

\end{proof}

\subsection{Proof of Theorem 2}
\label{proof2}

\begin{theorem}
\label{theorem:uppperbound}
To satisfy the condition in Eq.(\ref{expectation:condition}), the probability of having a branching factor $V{\geq} v$, $v{\geq}2$, should have the following upper bound:
\begin{equation}
    \textrm{Pr}(V{\geq}v) ~{\leq}~ \frac{e-2}{v-2}.
\end{equation}
\end{theorem}

\begin{proof}
When the average branching factor is equal to $e$, we can formalise the branching factor and the corresponding probability by:
\begin{equation}
\left\{
             \begin{array}{lr}
             \sum_{i=2}^{+\infty} i \cdot p_i&=e,  \\
             \sum_{i=2}^{+\infty} p_i&=1,  \\
             \end{array}
\right. 
\label{eqproof1}
\end{equation}
where $p_i$ corresponds to the probability to produce $i$ branches and the sum of the probabilities of producing different branches is equal to 1. To calculate the upper bond of $\textrm{Pr}(V{\geq}v)$, we should make $p_i=0$ for $2<i<v$, and then leverage the inequality:
\begin{equation}
\sum_{i=v}^{+\infty}ip_i ~{\geq}~ v\sum_{i=v}^{+\infty}p_i.
\label{eqproof2}
\end{equation}
Combining the condition that $p_i=0$ for $2<i<v$, we can rewrite Eq.(\ref{eqproof1}) as:
\begin{equation}
\left\{
             \begin{array}{lr}
            2p_2+ \sum_{i=v}^{+\infty} i \cdot p_i&=e,  \\
            p_2+ \sum_{i=v}^{+\infty} p_i&=1.  \\
             \end{array}
\right. 
\label{eqproof3}
\end{equation}
Then, we can have the following equation:
\begin{equation}
e-2=\sum_{i=v}^{+\infty} i p_i-2\sum_{i=v}^{+\infty} p_i, v > 2.
\label{eqproof4}
\end{equation}
By substituting Eq.(\ref{eqproof2}) into Eq.(\ref{eqproof4}), we can obtain:
\begin{equation}
e-2~{\geq} ~ \sum_{i=v}^{+\infty}ip_i-2\sum_{i=v}^{+\infty} p_i=(v-2)\sum_{i=v}^{+\infty} p_i, v > 2.
\label{eqproof5}
\end{equation}
Finally, we can have the following inequality relationship:
\begin{equation}
    \sum_{i=v}^{+\infty} p_i ~{\leq}~\frac{e-2}{v-2}, v > 2.
\end{equation}
Thus, the upper bound of $\textrm{Pr}(V{\geq}v)$ is $\frac{e-2}{v-2}$, $v>2$.

\end{proof}

\subsection{An Instantiation of the Probability Distribution of $V$}
\label{proof3}

Let a random variable $V$ denote the branching factor for an isolation tree $T$ with the sample space $\{v~|~v\in\mathbb{Z}~ \& ~v~{\geq}~2\}$. Let $\mathcal{D}$ be a distribution with the probability of taking value $v$ being: $\textrm{Pr}(V=v)=p_v$. We can generate the branching factors from the distribution $\mathcal{D}$. Then, we can have the following theorem with a proof.
\begin{proposition}
    If the distribution $\mathcal{D}$ is instantiated by assigning the probability of taking the branching factor $v$ as $p_v=\frac{(e-1)^2}{2e-1}e^{2-v}$, $v~{\geq}~2$, the resultant distribution satisfies the condition: $\mathbb{E}(V)=\sum_{v=2}^{+\infty}v{\cdot}p_v=e$.
\end{proposition}
\begin{proof}
\begin{equation}
\begin{aligned}
\mathbb{E}(V)&=\sum_{v=2}^{+\infty} v \cdot P_v  \\
    &=\sum_{v=2}^{+\infty} v \cdot \frac{e(e-1)^2}{2e-1} \cdot e^{1-v} \\
    &=\frac{e(e-1)^2}{2e-1} \sum_{v=2}^{+\infty} v \cdot e^{1-v}.
\label{eq21}
\end{aligned}
\end{equation}

{\noindent}For convenience, let $Z$ denote the quantity $\sum_{v=2}^{+\infty} v{\cdot}e^{1-v}$, i.e.,

\begin{equation}
Z \triangleq \sum_{v=2}^{+\infty} v \cdot e^{1-v}. 
\label{eq22}
\end{equation}
Multiplying both sides of the Eq.(\ref{eq22}) by $e^{-1}$ can result in:
\begin{equation}
e^{-1} \cdot Z = \sum_{v=2}^{+\infty} v \cdot e^{-v}. 
\label{eq23}
\end{equation}
Subtracting Eq.(\ref{eq23}) from Eq.(\ref{eq22}), we can derive the following result by using the formula of computing the sum of a geometric series with a common ratio $e^{-1}<1$.
\begin{equation}
\begin{aligned}
Z-e^{-1}Z&=2e^{-1}+e^{-2}+e^{-3}+...+e^{-v}+... \\
&=e^{-1}+(e^{-1}+e^{-2}+e^{-3}+...+e^{-v}+...) \\
&=e^{-1} + \frac{e^{-1}}{1-e^{-1}} \\
&=e^{-1} + \frac{1}{e-1}.
\label{eq24}
\end{aligned}
\end{equation}
According to Eq.(\ref{eq24}), the value of $Z$ can be calculated by:
\begin{equation}
Z = \frac{e^{-1} + \frac{1}{e-1}}{1-e^{-1}} = \frac{1+\frac{e}{e-1}}{e-1} = \frac{2e-1}{(e-1)^2}.
\label{eq25}
\end{equation}
Finally, the expectation $\mathbb{E}(V)$ can be calculated by Eq.(\ref{eq21}), (\ref{eq22}) and (\ref{eq25}):
\begin{equation}
\mathbb{E}(V) = \frac{e(e-1)^2}{2e-1} \cdot Z = \frac{e(e-1)^2}{2e-1} \cdot \frac{2e-1}{(e-1)^2} = e.
\label{eq26}
\end{equation}
\end{proof}

\section{Appendix B}
\label{appendix:baseline}

\subsection{Baselines}
\label{compared methods}
Our method OptIForest is compared with six state-of-the-art anomaly detection methods: iForest, LSHiForest, ECOD, REPEN, RDP, and DiForest, which are conventionally used for empirical study in many prior works. While there are many other anomaly detection methods, these six state-of-the-arts are chosen with the rationale that they are either very fresh or have shown superior performance in most cases. These methods are categorised into shallow anomaly detection methods and deep anomaly detection methods, briefly described as follows:
\begin{itemize}
	\item \textbf{Shallow Anomaly Detection Methods.} iForest \cite{liu2008isolation} is a seminal anomaly detection method, which can isolate data instances very efficiently with an ensemble of binary trees and usually achieves good performance. iForest has been widely recognised in academia and deployed in real applications, e.g., it has been included in \textit{scikit-learn}\footnote{\url{https://scikit-learn.org}}, a commonly-used machine learning library in Python. LSHiForest \cite{zhang2017lshiforest} is a generic framework, which generalise the forest isolation mechanism with the multi-fork tree structure and achieves higher performance and applicability. We select its instance L2SH which in general has the best performance with respect to accuracy and robustness for the comparison study. The source code is available in a public GitHub repository\footnote{\url{https://github.com/xuyun-zhang/LSHiForest}}. ECOD \cite{li2022ecod} is a very fresh anomaly detection method, which is parameter-free and easy to interpret. ECOD uses the empirical distributions of the input data to estimate tail probabilities per dimension for each data point. This method is simple but effective on many benchmark datasets. The source code is available in a public GitHub repository\footnote{\url{https://github.com/yzhao062/pyod}}.
	\item \textbf{Deep Anomaly Detection Methods.} REPEN \cite{pang2018learning} is a random distance-based anomaly detection method, which utilises deep representation learning and distance calculation to learn low-dimensional representations. REPEN has been widely accepted as a benchmark for deep anomaly detection and its source code is available in a public GitHub repository\footnote{\url{https://github.com/Minqi824/ADBench/tree/main/baseline}}. RDP \cite{wang2021unsupervised} trains neural networks to predict the abnormalities of data distances in a randomly projected space, where the genuine class structures are learned and implicitly embedded in the randomly projected space. RDP is a state-of-the-art deep anomaly detection method and achieves good performance on many datasets. The source code can be found in a public GitHub repository\footnote{\url{https://git.io/RDP}}. DiForest \cite{xu2022deep} is a very recent anomaly detection approach that utilises neural networks to learn representations and these representations are then used to construct isolation forests for anomaly detection. The source code is available in a public GitHub repository\footnote{\url{https://github.com/xuhongzuo/deep-iforest}}.
\end{itemize}


\subsection{Datasets}
\label{datasets}
We conduct experiments on 20 real-world datasets from different fields including finance, healthcare, network, etc. All datasets are available in public repositories like UCI Machine Learning Repository\footnote{\url{https://archive.ics.uci.edu/ml/datasets.php}}, Kaggle Repository\footnote{\url{https://www.kaggle.com/datasets}}, and ADRepository\footnote{\url{https://github.com/GuansongPang/ADRepository-Anomaly-detection-datasets}}.
The basic information about the datasets is summarised in Table~\ref{tab:Real-world datasets}.
\begin{table}[!t]
	\centering
        
	\caption{A summary of datasets used in the experiments. Here, ``$n$'' in the table denotes the size of the dataset, ``$m$'' denotes the dimension of the dataset, and ``Rate'' represents the proportion of anomalous data in total data instances.}
	\label{tab:Real-world datasets}
        \resizebox{0.45\textwidth}{!}{%
	\begin{tabular}{l l l l l} 
		\toprule
		Dataset & $n$ & $m$ & Rate(\%) & Category\\ 
		\midrule
		AD & 3,279 & 1,555 & 13.79 & Finance \\
		campaign & 41,188 & 62 & 11.27 & Finance \\
		Arrhythmia & 452 & 274 & 14.60 & Healthcare \\
		cardio & 1,831 & 21 & 9.61 & Healthcare \\
		backdoor & 95,329 & 196 & 2.44 & Network \\
		KDDCup99 & 494,021 & 38 & 1.77 & Network \\
		Celeba & 202,599 & 39 & 2.24 & Image \\
		mnist & 7,603 & 100 & 9.21 & Image \\
		Census & 299,285 & 500 & 6.20 & Sociology \\
		Donors & 619,326 & 10 & 5.92 & Sociology \\
		Cover & 286,048 & 10 & 0.96 & botany \\
		http & 567,498 & 3 & 0.39 & Web \\
		smtp & 95,156 & 3 & 0.03 & Web \\
		Ionosphere & 351 & 34 & 35.90 & Oryctognosy \\
		Satellite & 6,435 & 36 & 31.60 & Astronautics \\
		Shuttle & 49,097 & 9 & 7.15 & Astronautics \\
		Spam & 4,207 & 57 & 39.91 & Document \\
		Vowel & 1,456 & 12 & 3.43 & Linguistics \\
        Waveform & 3,505 & 21 & 4.62 & Physics \\
		Wine & 5,318 & 11 & 4.53 & Chemistry \\
		\bottomrule
	\end{tabular}
 }
\end{table}

\section{Appendix C}
\label{appendix:setup}

\subsection{Parameter Settings}
\label{setting}
Our method uses 100 isolation trees as the base detector for the isolation forest, the same as the existing isolation forest based methods. In section \ref{sect:Ablation}, we study the relationship between the sampling size and the detection performance and observe the performance of our method is saturated when the sampling size reaches a threshold ($2^9 = 512$ herein). Thus, the size of the sample used for constructing each tree is 512 in our method. Besides, we study how the cut threshold influence the detection performance in section \ref{sect:Ablation} and observe that the selection of cut threshold $\epsilon$ is influenced by the size of dataset, i.e., the large datasets perform well with a big cut threshold (e.g., $\epsilon=403$), while the small datasets perform well with a small cut threshold (e.g., $\epsilon=55$). To make a fair comparison, our method will use the above optimal settings to get the results. In the compared methods, all the parameters are set as the optimal settings in the original papers.
\subsection{Evaluation Metrics}
\label{metric}
We use the Area Under Receiver Operating Characteristic Curve (AUC-ROC) and Area Under the Precision-Recall Curve (AUC-PR) as the accuracy evaluation criterion \cite{kurt2020real,wang2021unsupervised}. Before explaining two evaluation metrics, we need to introduce four indicators derived from the confusion matrix, including Recall, Precision, True Positive Rate (TPR), and False Positive Rate (FPR). These four indicators can be calculated as:
\begin{equation*}
    \textrm{Recall}=\frac{\textrm{TP}}{\textrm{TP+FN}},
\end{equation*}
\begin{equation*}
    \textrm{Precision}=\frac{\textrm{TP}}{\textrm{TP+FP}},
\end{equation*}
\begin{equation*}
    \textrm{True Positive Rate}=\frac{\textrm{TP}}{\textrm{TP+FN}},
\end{equation*}
\begin{equation*}
    \textrm{False Positive Rate}=\frac{\textrm{FP}}{\textrm{FP+TN}},
\end{equation*}

\noindent where $T$ represents the original object or event is true, $F$ represents the original object or event is false, $P$ represents the object is predicted as a positive example by the classifier, $N$ represents the object is predicted as a negative example by the classifier. Then, $TP$ can be explained as the number of true objects that are predicted to be positive by the classifier. The total number of data instances, denoted as $S$, is formalised as $S=TP+FP+TN+FN$. The ROC curve takes TPR as the Y-axis and FPR as the X-axis, so the value of AUC-ROC is the area under the ROC curve. Similarly, the PR curve summarises the relation between Precision and Recall. The value of AUC ranges from 0 to 1, where a larger AUC result indicates better performance. The result of AUC has been extensively adopted in many anomaly detection works and has become an essential measure of accuracy in correlational research. Furthermore, the execution time is used as the efficiency evaluation criterion. All experiments are run 15 times and averaged results are reported.

\section{Appendix D}
\subsection{Execution Time}
\label{execution time}

We use execution time as the efficiency metric to compare OptIForest with other methods on 20 real world datasets , and the results are reported in Table \ref{tab:time}. It can be seen in Table \ref{tab:time} that our method OptIForest has much shorter execution time than the deep anomaly detetion methods of REPEN and RDP for most datasets. It is worth noting that the more our method learns, the more time it will spend. As previously discussed in the AUC comparison, achieving optimal performance on small datasets necessitates more learning, resulting in longer execution time compared to REPEN on some small datasets. As expected, OptIForest reasonably has longer execution time than other isolation forests where no learning is conducted. So, it can be concluded that our approach strikes a good balance between execution efficiency and detection performance. It is appealing that OptIForest can achieve better performance than deep learning methods but takes much less execution cost.
\begin{table}[!t]
	\centering
	\caption{Comparing execution time (s) of all methods. It is worth noting that OptIForest has much shorter execution time than the deep learning methods of REPEN and RDP for most datasets.}
	\label{tab:time}
        \resizebox{0.5\textwidth}{!}{%
	\begin{tabular}{llllllll} 
		\toprule
		Dataset & iForest & LSHiForest & ECOD & REPEN & RDP & DiForest & \textbf{OptIForest} \\ 
		\midrule
		AD & 10 & 26 & 4 & 23 & 7,327 & 14 & 83 \\
		campaign & 7 & 294 & 3 & 584 & 7,384 & 111 & 637 \\
		Arrhythmia & 0.3 & 6 & 0.2 & 16 & 6,037 & 2 & 58 \\
		cardio & 0.3 & 26 & 0.1 & 9 & 6,508 & 6 & 68 \\
		backdoor & 26 & 758 & 18 & 2,647 & 5,267 & 245 & 1,289 \\
		KDDCup99 & 40 & 2,409 & 18 & 38,561 & 6,991 & 634 & 5,652 \\
		celeba & 22 & 1,055 & 7 & 14,334 & 8,026 & 579 & 2,622 \\
		mnist & 2 & 35 & 1 & 43 & 7,482 & 22 & 132 \\		
		census & 205 & 1,911 & 185 & 30,794 & 8,991 & 883 & 4,773 \\
		donors & 24 & 3,901 & 7 & 51,496 & 7,655 & 1,367 & 6,784 \\
		Cover & 14 & 1,474 & 6 & 27,930 & 6,472 & 732 & 3,637 \\
		http & 19 & 3,557 & 4 & 91,980 & 6,098 & 1,322 & 7,083\\
		smtp & 4 & 501 & 1 & 7,115 & 6,422 & 203 & 1,262 \\
		Ionosphere & 0.3 & 6 & 0.1 & 3 & 3,787 & 2 & 69 \\
		Satellite & 1 & 36 & 0.4 & 213 & 4,649 & 17 & 209 \\
		Shuttle & 2 & 313 & 1 & 1,060 & 4,928 & 124 & 662 \\
		Spam & 1 & 15 & 0.3 & 123 & 6,316 & 10 & 83 \\
		Vowel & 0.3 & 12 & 0.1 & 60 & 3,743 & 5 & 50 \\
		Waveform & 0.5 & 17 & 0.2 & 112 & 3,839 & 10 & 109 \\
		Wine & 0.4 & 29 & 0.1 & 156 & 4,613 & 15 & 126 \\
		\bottomrule
	\end{tabular}
 }
\end{table}

\end{document}